\documentclass{article}
\usepackage[nonatbib,preprint]{neurips_2024}

\usepackage[utf8]{inputenc} 
\usepackage[T1]{fontenc}    
\usepackage{url}            
\usepackage{booktabs}       
\usepackage{amsfonts}       
\usepackage{nicefrac}       
\usepackage{microtype}      

\usepackage{algorithm}
\usepackage{algorithmicx}
\usepackage[noend]{algpseudocode}
\usepackage{amsmath}
\usepackage{amssymb}
\usepackage{amsthm}
\usepackage{bbm}
\usepackage{bm}
\usepackage{color}
\usepackage{dirtytalk}
\usepackage{dsfont}
\usepackage{enumerate}
\usepackage{float}
\usepackage{graphicx}
\usepackage{listings}
\usepackage{mathtools}
\usepackage[numbers]{natbib}
\usepackage{subcaption}
\usepackage{times}
\usepackage{xspace}

\usepackage[usenames,dvipsnames]{xcolor}
\usepackage[bookmarks=false]{hyperref}
\hypersetup{
  pdffitwindow=true,
  pdfstartview={FitH},
  pdfnewwindow=true,
  colorlinks,
  linktocpage=true,
  urlcolor=Green,
  linkcolor=Green,
  citecolor=Green
}
\usepackage[capitalize,noabbrev]{cleveref}

\usepackage[textsize=tiny,disable]{todonotes}

\newcommand{\todol}[2][]{\todo[color=blue!20,size=\tiny,#1]{L: #2}} 

\usepackage{thmtools}
\declaretheorem[name=Theorem,refname={Theorem,Theorems},Refname={Theorem,Theorems}]{theorem}
\declaretheorem[name=Lemma,refname={Lemma,Lemmas},Refname={Lemma,Lemmas},sibling=theorem]{lemma}

\newcommand{\cA}{\mathcal{A}}

\newcommand{\cN}{\mathcal{N}}

\newcommand{\cS}{\mathcal{S}}

\newcommand{\cX}{\mathcal{X}}

\newcommand{\realset}{\mathbb{R}}

\newcommand{\E}[1]{\mathbb{E} \left[#1\right]}
\newcommand{\condE}[2]{\mathbb{E} \left[#1 \,\middle|\, #2\right]}
\newcommand{\Erv}[2]{\mathbb{E}_{#1} \left[#2\right]}

\newcommand{\condprob}[2]{\mathbb{P} \left(#1 \,\middle|\, #2\right)}

\newcommand{\I}[1]{\mathds{1} \! \left\{#1\right\}}

\newcommand{\T}{^\top}

\DeclareMathOperator*{\argmax}{arg\,max\,}

\DeclareMathOperator{\sgn}{sgn}

\mathchardef\mhyphen="2D

\newcommand{\setdr}{\ensuremath{\tt SetDR}\xspace}
\newcommand{\setips}{\ensuremath{\tt SetIPS}\xspace}

\title{Off-Policy Evaluation from Logged Human Feedback}

\author{
  Aniruddha Bhargava$^*$ \\
  Amazon
  \And
  Lalit Jain$^*$ \\
  University of Washington
  \And
  Branislav Kveton$^*$ \\
  AWS AI Labs
  \And
  Ge Liu$^*$ \\
  University of Illinois Urbana-Champaign
  \And
  Subhojyoti Mukherjee$^*$ \\
  University of Wisconsin–Madison
}

\begin{document}

\maketitle

\begin{abstract}
Learning from human feedback has been central to recent advances in artificial intelligence and machine learning. Since the collection of human feedback is costly, a natural question to ask is if the new feedback always needs to collected. Or could we evaluate a new model with the human feedback on responses of another model? This motivates us to study off-policy evaluation from logged human feedback. We formalize the problem, propose both model-based and model-free estimators for policy values, and show how to optimize them. We analyze unbiasedness of our estimators and evaluate them empirically. Our estimators can predict the absolute values of evaluated policies, rank them, and be optimized.
\end{abstract}

\section{Introduction}
\label{sef:introduction}

\emph{Large language models (LLMs)} \citep{bommasani21opportunities} have recently emerged as general purpose inference machines that achieve human-level performance on a wide range of tasks \citep{brown20language,mirchandani2023large}. The key step in training them is \emph{reinforcement learning with human feedback (RLHF)}, where these models are aligned to generate human-preferred text \citep{ouyang22training}. The key idea in RLHF is to use human feedback to learn a latent reward model. After this, a policy is optimized to maximize the reward under the reward model. In DPO \citep{rafailov23direct}, the reward model is reparameterized using the policy, which is then optimized. One common property of RLHF and DPO is their reliance on human feedback. After each alignment, a new dataset with human feedback is generated. A natural question to ask is: do new datasets always need to be collected or could we evaluate a new LLM with human feedback on responses of another LLM?

\renewcommand{\thefootnote}{\fnsymbol{footnote}}
\footnotetext[1]{This work was done outside of Amazon. The author names are listed alphabetically.}
\renewcommand{\thefootnote}{\arabic{footnote}}

Motivated by this question, we study off-policy evaluation \citep{li10contextual,bottou13counterfactual,li15counterfactual,hofmann16online} with human feedback. We formulate the problem as follows. Given a logged dataset of $n$ lists of responses, which are generated by an LLM and re-ranked by humans according to their preferences, we want to estimate how another LLM would align with human feedback without another human study. We say that the LLM aligns with human feedback when its first response is the same as the most preferred human response. We envision two main use cases for our approach. The first use case is counterfactual evaluation \citep{bottou13counterfactual,li15counterfactual}, akin to those in learning to rank \citep{joachims17unbiased,swaminathan17offpolicy,li18offline}. The second use case is LLM alignment \citep{ouyang22training,rafailov23direct} with an interpretable objective that represents counterfactual human feedback. To allow these, we make algorithmic contributions in off-policy evaluation and optimization, and summarize them below:
\begin{enumerate}
  \item We formalize the problem of off-policy evaluation from logged human feedback as offline evaluation with ranked lists \citep{joachims17unbiased,swaminathan17offpolicy,li18offline}. Specifically, given a logged dataset of $n$ lists of length $K$, which are generated by the logging policy and re-ranked by humans according to their preferences, we want to estimate how another policy would align with human feedback without another human study. The novelty in our work is in a new feedback model induced by the \emph{Plackett-Luce (PL) model} \citep{plackett75analysis,luce05individual,zhu23principled}. The \emph{Bradley-Terry-Luce (BTL) model} \citep{bradley52rank,sekhari2024contextual} is a special case of $K = 2$. From the feedback point of view, our problem is both a bandit \citep{lattimore19bandit}, because only $K$ responses are ranked by a human; and full-information \citep{auer95gambling}, because the full ranking of the $K$ responses is observed.
  \item We propose model-based \citep{robins94estimation,dudik14doubly} and model-free \citep{horwitz52generalization,strehl10learning} estimators. The model-based estimator relies on a learned reward model, similarly to RLHF. However, unlike in RLHF, the mean reward is the probability that the policy aligns with human feedback, instead of being latent. This makes our estimator interpretable. The model-free approach is based on \emph{inverse propensity scores (IPS)} \citep{horwitz52generalization,ionides08truncated}. We improve upon a naive application of IPS over ranked lists by IPS over sets of responses (\cref{sec:advanced ips}). This is only possible because of the PL feedback model, which depends on the set of ranked responses but not their order. We analyze basic properties of our estimators, such as unbiasedness. We also show how to optimize them in a practical way \citep{swaminathan15counterfactual}.
  \item We comprehensively evaluate our estimators in multiple experiments. First, we show that they can estimate the absolute value of evaluated policies. Second, we show that they can rank policies better than the latent reward functions that RLHF and DPO optimize. Third, we confirm our findings on a real-world dataset where the reward model is misspecified. Finally, we show that optimization of our estimators leads to comparable policies to those learned by RLHF and DPO.
\end{enumerate}

This paper is organized as follows. In \cref{sec:setting}, we introduce the problem of off-policy evaluation from logged human feedback. In \cref{sec:off-policy evaluation}, we present our model-based and model-free estimators. In \cref{sec:off-policy optimization}, we show how to optimize our estimators. In \cref{sec:experiments}, we empirically evaluate our estimators. We review prior works in \cref{sec:related work} and conclude in \cref{sec:conclusions}.

\section{Setting}
\label{sec:setting}

We study the following setting. A policy, for instance given by an LLM, interacts with a human for $n$ rounds. In round $t \in [n]$, the policy generates a ranked list of $K$ responses $A_t$ to the query in round $t$. The human critiques $A_t$ and provides their preferred order of the responses $A_{t, *}$, represented as a permutation of $A_t$. The policy \emph{aligns with human feedback} in round $t$ when $A_t$ and $A_{t, *}$ are \emph{similar}. We want to reuse previously logged data to estimate alignment of another policy without collecting additional human feedback.

We formalize our problem as follows. The query in round $t$ is $x_t \in \cX$, where $\cX$ is the query set. There are $L$ potential responses to any query and we denote the set of integers that indexes them by $\cA = [L]$. We do not assume that $L$ is small and comment on its impact on our estimators throughout the paper. A logging policy $\pi_0$ generates a ranked list of $K$ responses $A_t = (a_{t, i})_{i = 1}^K$ from $\cA$, where $a_{t, i} \in \cA$ is the $i$-th response. The human responds with a permutation of $A_t$, $A_{t, *} = (a_{t, *, i})_{i = 1}^K$, that represents their preferred order of the responses.

The setting of $K < L$ is motivated by the limited capacity of humans to provide preferential feedback on too many choices $L$. When $K = 2$, we have a relative feedback over two responses, as in the Bradley-Terry-Luce model \citep{bradley52rank}. When $K > 2$, we have a ranking feedback over $K$ responses, as in the Plackett-Luce model \citep{plackett75analysis,luce05individual}. We define the policies and human alignment next.

\textbf{Policy.} We denote by $\pi(a \mid x)$ the probability that policy $\pi$ generates a response $a \in \cA$ to query $x \in \cX$. A ranked list of $K$ responses $A = (a_i)_{i = 1}^K$ is sampled from $\pi(\cdot \mid x)$ with probability
\begin{align}
  \pi(A \mid x)
  = \prod_{i = 1}^K \frac{\pi(a_i \mid x)}
  {\sum_{j = i}^L \pi(a_j \mid x)}\,,
  \label{eq:policy}
\end{align}
where $a_{k + 1}, \dots, a_L$ are arbitrarily ordered responses from $\cA \setminus A$. In plain English, the first response is sampled with probability $\pi(a_1 \mid x)$, the second with probability $\pi(a_2 \mid x) / \sum_{j = 2}^L \pi(a_j \mid x)$, and the $i$-th with probability $\pi(a_i \mid x) / \sum_{j = i}^L \pi(a_j \mid x)$. This is similar to the PL model \citep{plackett75analysis,luce05individual}.

The probabilities of individual responses can be obtained from an LLM as follows. Take $L$ most frequent responses to each query $x$ and let $\tilde{\pi}(a \mid x)$ be the probability (potentially unnormalized) of the tokens corresponding to response $a$. Then $\pi(a \mid x) = \tilde{\pi}(a \mid x) / \sum_{a' \in \cA} \tilde{\pi}(a' \mid x)$.

\textbf{Reward.} The alignment of the policy with human feedback can be defined in many ways. In this work, we say that the policy \emph{aligns with human feedback} in round $t$ when the first responses in $A_t$ and $A_{t, *}$ are identical. We represent the alignment using a numerical reward. Specifically, the \emph{reward} in round $t$ is $\I{a_{t, 1} = a_{t, *, 1}}$ and the \emph{mean reward} is
\begin{align}
  r(x_t, A_t)
  = \condprob{a_{t, 1} = a_{t, *, 1}}{x_t, A_t}\,,
  \label{eq:reward}
\end{align}
where $A_{t, *}$ is the human-preferred order of responses $A_t$. We wanted to comment on two aspects of the mean reward. First, note that $A_{t, *}$ is a random variable that depends on $A_t$ and $x_t$. Second, the mean reward is essentially the probability that the most preferred human response $a_{t, *, 1}$ is the first logged response $a_{t, 1}$.

The mean reward in \eqref{eq:reward} can be justified from many points of view. First, it is the probability that the first option out of $K$ is chosen in a discrete choice model \citep{train09discrete,benakiva18discrete}, a broad class of classic models of human preferences. Second, it is analogous to Precision@K in ranking \citep{manning08introduction} for $K = 1$, where the human feedback is the ground truth and the policy is a ranker. Finally, as we shall see in \cref{sec:off-policy evaluation}, this quantity can be estimated in a model-free fashion from our feedback, both on- and off-policy. For concreteness, one can think of the mean reward as
\begin{align}
  r(x, A; w)
  = \frac{\exp[\phi(x, a_1)\T w]}{\sum_{i = 1}^K \exp[\phi(x, a_i)\T w]}\,,
  \label{eq:reward model}
\end{align}
where $A = (a_i)_{i = 1}^K$ is a ranked list of $K$ responses, $\phi: \cX \times \cA \to \realset^d$ is a feature map, and $w \in \realset^d$ is a reward model parameter. We denote the true unknown parameter  by $w_*$. A natural way of thinking of \eqref{eq:reward model} is as the probability that the first response $A_1$ is preferred over the rest. The algebraic form is borrowed from the PL model. This has two implications. First, $w$ can be estimated using existing techniques for the PL model \citep{zhu23principled}. Second, the probability depends on all responses in $A$ but not their order. We exploit this property in the design of our more advanced estimators in \cref{sec:advanced ips}.

With the definitions of policies and human-feedback alignment in hand, we can define the value of a policy. The value of policy $\pi$ is the probability that its first response aligns with the human-preferred response. Formally, over $n$ rounds with queries $(x_t)_{t = 1}^n$, this is
\begin{align}
  V(\pi)
  = \frac{1}{n} \sum_{t = 1}^n \Erv{A_t \sim \pi(\cdot \mid x_t)}{r(x_t, A_t)}\,.
  \label{eq:policy value}
\end{align}
Our goal is to estimate this quantity from human feedback on responses generated by another policy $\pi_0$. Note that we could have assumed that $x_t \sim p$ for some distribution $p$. We do not because the stochasticity of queries is not necessary to derive any result in this work.

\section{Off-Policy Evaluation}
\label{sec:off-policy evaluation}

The key idea in our work is to pose the problem of off-policy evaluation from logged human feedback as a counterfactual evaluation problem \citep{li10contextual,bottou13counterfactual,li15counterfactual,hofmann16online} over ranked lists \citep{joachims17unbiased,swaminathan17offpolicy,li18offline}. Before we get to the details, note that this problem is particularly easy when the human feedback is collected using policy $\pi$. Then an unbiased estimate of $V(\pi)$ in \eqref{eq:policy value} is the frequency that the first responses in $A_t$ and $A_{t, *}$ are identical,
\begin{align}
  \hat{V}(\pi)
  = \frac{1}{n} \sum_{t = 1}^n \I{a_{t, 1} = a_{t, *, 1}}\,.
  \label{eq:counter}
\end{align}
We prove this in \cref{lem:counter} (\cref{sec:technical lemmas}). Now suppose that $A_t \sim \pi_0(\cdot \mid x_t)$, where $\pi_0$ is the logging policy. Then $V(\pi)$ can be estimated using either model-based or model-free techniques.

\subsection{Direct Method}
\label{sec:dm}

A popular approach to off-policy evaluation is the \emph{direct method (DM)} \citep{dudik14doubly}. The key idea in the DM is to learn a reward model and then use it to estimate the expected value of a policy. A natural choice in our setting is \eqref{eq:reward model}. We estimate $w$ by solving the maximum likelihood estimation (MLE) problem
\begin{align}
  \hat{w}
  = \argmax_w \sum_{t = 1}^n \sum_{i = 1}^K
  \log\left(\frac{\exp[\phi(x_t, a_{t, *, i})\T w]}
  {\sum_{j = i}^K \exp[\phi(x_t, a_{t, *, j})\T w]}\right)\,.
  \label{eq:reward model estimation}
\end{align}
Note that this estimator is defined over $K$ responses while the reward model in \eqref{eq:reward model} is the probability of the first response over the rest. This is not harmful since all sampling stages in the PL model share the same parameter. With the reward model in hand, we estimate the value of policy $\pi$ as
\begin{align}
  \hat{V}_\textsc{dm}(\pi)
  = \frac{1}{n} \sum_{t = 1}^n
  \Erv{A \sim \pi(\cdot \mid x_t)}{r(x_t, A; \hat{w})}\,,
  \label{eq:dm}
\end{align}
where $r(x, A; \hat{w})$ is the estimated preference for response $a_1$ from $A$. Also note that \eqref{eq:reward model estimation} is analogous to learning the reward model in RLHF (\cref{sec:related work}). The difference in \eqref{eq:dm} is in how the reward model is used. In \eqref{eq:rlhf}, the latent preference for a single human response $a$ is estimated. In \eqref{eq:dm}, we estimate the probability that the first response in $A$ is preferred by the human.

The strength of the DM estimator is that the policy $\pi$ only needs to be sampled from. The probability $\pi(A \mid x_t)$, which requires normalization and thus a summation over $L$ terms in \eqref{eq:policy}, is not needed. On the other hand, the estimator may perform poorly when the reward model is biased.

\subsection{Propensity-Based Methods}
\label{sec:ips}

Another popular approach to off-policy evaluation are \emph{inverse propensity scores (IPS)} \citep{horwitz52generalization,ionides08truncated}. The key idea in IPS is to reweigh the data collected by the logging policy $\pi_0$ as if they were logged by the evaluated policy $\pi$. Specifically, let $A_t \sim \pi(\cdot \mid x_t)$. Then based on the definition of the expected value of a policy in \eqref{eq:policy value}, the IPS estimator is
\begin{align}
  \hat{V}_\textsc{ips}(\pi)
  = \frac{1}{n} \sum_{t = 1}^n
  \frac{\pi(A_t \mid x_t)}{\pi_0(A_t \mid x_t)}
  \I{a_{t, 1} = a_{t, *, 1}}\,.
  \label{eq:ips}
\end{align}
We prove that this estimator is unbiased in \cref{lem:ips} (\cref{sec:technical lemmas}).

The strength of the IPS estimator is that it does not make any assumptions about the reward model. It only reweighs the logged rewards $\I{a_{t, 1} = a_{t, *, 1}}$ by $\pi(A_t \mid x_t) / \pi_0(A_t \mid x_t)$. The estimator has two shortcomings. First, its variance can be high when the policies $\pi$ and $\pi_0$ are not close, because $\pi(A_t \mid x_t) / \pi_0(A_t \mid x_t)$ is high. Second, the computation of $\pi(A \mid x_t)$ requires normalization over $L$ terms in \eqref{eq:policy}, which may be challenging when $L$ is large.

The \emph{doubly-robust method (DR)} \citep{robins94estimation,dudik14doubly,jiang2016doubly} combines the advantages of the DM and IPS. Specifically, it uses the model as a variance-reduction techniques in the IPS. For the DM and IPS estimators in \eqref{eq:dm} and \eqref{eq:ips}, respectively, it is
\begin{align}
  \hat{V}_\textsc{dr}(\pi)
  = \frac{1}{n} \sum_{t = 1}^n
  \frac{\pi(A_t \mid x_t)}{\pi_0(A_t \mid x_t)}
  (\I{a_{t, 1} = a_{t, *, 1}} - r(x_t, A_t; \hat{w})) + \hat{V}_\textsc{dm}(\pi)\,.
  \label{eq:dr}
\end{align}
The DR estimator is unbiased when the DM is unbiased or the propensities in the IPS estimator are correctly specified. We prove this in \cref{lem:dr} (\cref{sec:technical lemmas})

The DR estimator tends to have a lower variance than the IPS estimator when the original rewards $\I{a_{t, 1} = a_{t, *, 1}}$ have high variances but the centered rewards $\I{a_{t, 1} = a_{t, *, 1}} - r(x_t, A_t; \hat{w})$ do not. The estimator also inherits the computational limitation of IPS, that the normalization over $L$ terms in \eqref{eq:policy} is needed to compute the propensity scores.

\subsection{Advanced Propensity-Based Methods}
\label{sec:advanced ips}

The additional structure of our problem allows us to improve the IPS estimator in \eqref{eq:ips}. To define the improved estimator, we need additional notation. Let $S_t$ be the set of responses in $A_t$ forgetting the ranking and $\nu(S_t \mid x_t)$ be the probability that the set $S_t$ is generated in round $t$ by the policy $\pi$. The quantities $\pi(A \mid x)$ and $\nu(S \mid x)$ are related as
\begin{align}
  \textstyle
  \nu(S \mid x)
  = \sum_{A \in \Pi(S)} \pi(A \mid x)\,,
  \label{eq:lists to sets}
\end{align}
for any $x$ and $S$, where $\Pi(S)$ is the set of all permutations of $S$. Moreover, let
\begin{align*}
  \pi(a \mid S, x)
  = \frac{\pi(a \mid x)}{\sum_{a' \in S} \pi(a' \mid x)}
\end{align*}
be the probability that policy $\pi$ generates a response $a$ from set $S$ to query $x$.

The key insight is that the human feedback tells us which response in $A_t$ is preferred. As a result, we can do counterfactual reasoning conditioned on $x_t$ and $S_t$. Specifically, let $\cS$ be the set of all subsets of $\cA$ of size $K$. Then the value of policy $\pi$ in \eqref{eq:policy value} can be rewritten as
\begin{align*}
  V(\pi)
  & = \frac{1}{n} \sum_{t = 1}^n \sum_{S_t \in \cS} \sum_{A_t \in \Pi(S_t)}
  \pi(A_t \mid x_t) \, r(x_t, A_t) \\
  & = \frac{1}{n} \sum_{t = 1}^n \sum_{S_t \in \cS} \nu(S_t \mid x_t)
  \sum_{A_t \in \Pi(S_t)} \frac{\pi(A_t \mid x_t)}{\nu(S_t \mid x_t)} \, r(x_t, A_t) \\
  & = \frac{1}{n} \sum_{t = 1}^n \sum_{S_t \in \cS} \nu(S_t \mid x_t) \,
  \condE{\I{a_{t, 1} = a_{t, *, 1}}}{x_t, S_t} \\
  & = \frac{1}{n} \sum_{t = 1}^n \sum_{S_t \in \cS} \nu(S_t \mid x_t) \,
  \condE{\sum_{a \in \cA} \pi(a \mid S_t, x_t) \, \I{a_{t, *, 1} = a}}{x_t, S_t} \\
  & = \frac{1}{n} \sum_{t = 1}^n \sum_{S_t \in \cS} \nu(S_t \mid x_t) \,
  \condE{\pi(a_{t, *, 1} \mid S_t, x_t)}{x_t, S_t} \\
  & = \frac{1}{n} \sum_{t = 1}^n
  \Erv{S_t \sim \nu(\cdot \mid x_t)}{\pi(a_{t, *, 1} \mid S_t, x_t)}\,.
\end{align*}
The fourth equality follows from the assumption that $A_{t, *}$ depends only on the set of responses $S_t$. The expectation in the last step is needed because $A_{t, *}$ remains to be a random variable.

Now note that the process of sampling ranked lists can also be viewed as sampling sets. Therefore, when $S_t \sim \nu(\cdot \mid x_t)$, $V(\pi)$ can be estimated as
\begin{align}
  \hat{V}_\textsc{set}(\pi)
  = \frac{1}{n} \sum_{t = 1}^n \pi(a_{t, *, 1} \mid S_t, x_t)\,.
  \label{eq:soft counter}
\end{align}
We prove that this estimator is unbiased in \cref{lem:soft counter} (\cref{sec:technical lemmas}). The main improvement over \eqref{eq:counter} is due to reducing variance, by eliminating the randomness in $A_t \mid S_t$. \todol{I think the proof of this is probably immediate. I'll try to add something to the appendix.} Similarly to the IPS estimator in \eqref{eq:ips}, when $S_t \sim \nu_0(\cdot \mid x_t)$, the IPS estimator becomes
\begin{align}
  \hat{V}_\textsc{set-ips}(\pi)
  = \frac{1}{n} \sum_{t = 1}^n
  \frac{\nu(S_t \mid x_t)}{\nu_0(S_t \mid x_t)}
  \pi(a_{t, *, 1} \mid S_t, x_t)\,.
  \label{eq:set ips}
\end{align}
We call it \setips because the propensity scores are over sets of responses instead of ranked lists. We prove that this estimator is unbiased in \cref{lem:set ips} (\cref{sec:technical lemmas}). 

The main improvement in \eqref{eq:set ips} over \eqref{eq:ips} is due to reducing variance. In particular, the propensities over a larger set, of all lists of length $K$, are replaced with a smaller set, of all sets of size $K$. Due to the relation in \eqref{eq:lists to sets}, this leads to higher propensities and therefore a better control of their ratios. The new estimator has an interesting behavior as $K \to L$, and in particular $K = L$. In the latter case, the propensities vanish because $\nu(S_t \mid x_t) = \nu_0(S_t \mid x_t) = 1$. The classic IPS estimator in \eqref{eq:ips} does not poses this property and would be cursed by low propensities. \setips also inherits the limitations of IPS, that the variance can be high and that the normalization over $L$ terms in \eqref{eq:policy} can be challenging.

Similarly to \eqref{eq:dr}, we can define a \emph{set doubly-robust method (\setdr)}. Specifically, the probability of aligning with human feedback under policy $\pi$ given set $S$ and reward model $r(x, A; w)$ is
\begin{align*}
  r(x, S; \pi, w)
  = \sum_{A \in \Pi(S)} \frac{\pi(A \mid x)}{\nu(S \mid x)} r(x, A; w)\,.
\end{align*}
When this model is used as a variance-reduction techniques in \setips, we get
\begin{align}
  \hat{V}_\textsc{set-dr}(\pi)
  = \frac{1}{n} \sum_{t = 1}^n
  \frac{\nu(S_t \mid x_t)}{\nu_0(S_t \mid x_t)}
  (\pi(a_{t, *, 1} \mid S_t, x_t) - r(x_t, S_t; \pi, \hat{w})) +
  \hat{V}_\textsc{dm}(\pi)\,.
  \label{eq:set dr}
\end{align}
The DR estimator is unbiased when the DM is unbiased or the propensities in the IPS estimator are correctly specified. We prove this in \cref{lem:set dr} (\cref{sec:technical lemmas})

\section{Off-Policy Optimization}
\label{sec:off-policy optimization}

The strength of our approach is that any estimator from \cref{sec:off-policy evaluation} can be used for optimization. In particular, let $\pi(a \mid x; \theta)$ be a parameterization of policy $\pi$ by $\theta \in \realset^d$. Then the optimization of its estimated value can be viewed as maximizing
\begin{align}
  \hat{V}(\pi) - \gamma \frac{1}{n} \sum_{t = 1}^n
  d(\pi(\cdot \mid x_t; \theta), \pi_0(\cdot \mid x_t))\,,
  \label{eq:optimized objective}
\end{align}
where $d(p, q)$ is the KL divergence between distributions $p$ and $q$ with support $\cA$ and $\gamma > 0$ is a tunable parameter. The KL term can be viewed as a constraint that forces $\pi$ to be close to $\pi_0$, and plays the same role as in RLHF and DPO (\cref{sec:related work}). The gradient of \eqref{eq:optimized objective} with respect to $\theta$ is
\begin{align}
  \nabla \hat{V}(\pi) - \gamma \frac{1}{n} \sum_{t = 1}^n
  \nabla d(\pi(\cdot \mid x_t; \theta), \pi_0(\cdot \mid x_t))\,.
  \label{eq:optimized gradient}
\end{align}
We derive $\nabla \hat{V}(\pi)$ for all of our estimators in \cref{sec:gradients}. To compute the sum, we sample $\tilde{A}_t \sim \pi(\cdot \mid x_t; \theta)$ and replace $\nabla d(\pi(\cdot \mid x_t; \theta), \pi_0(\cdot \mid x_t))$ with its unbiased estimate
\begin{align*}
  (\nabla \log \pi(\tilde{A}_t \mid x_t; \theta))
  (1 + \log \pi(\tilde{A}_t \mid x_t; \theta) - \log \pi_0(\tilde{A}_t \mid x_t))\,.
\end{align*}
We properly derive this gradient in \cref{sec:gradients}. As in RLHF, the normalization over $L$ terms in \eqref{eq:policy} is needed to compute the above log-probabilities.

\section{Experiments}
\label{sec:experiments}

We conduct four experiments. In \cref{sec:absolute error}, we evaluate our off-policy estimators on predicting the absolute value of a policy. In \cref{sec:relative error}, we evaluate them on ranking policies. In \cref{sec:llms}, we apply our estimators to large language models. Finally, in \cref{sec:policy optimization}, we optimize our off-policy estimators. In all plots but in \cref{fig:policy optimization}, we report standard errors of the estimates.

\subsection{Absolute Error}
\label{sec:absolute error}

We first evaluate how good our off-policy estimators are in predicting the absolute value of a policy. This experiment is conducted on synthetic problems, which are generated as follows. The number of potential responses is $L = 7$. We choose this value because the maximum number of unique ranked lists $7! = 5\,040$ when $K = L$. This is more than enough to illustrate the impracticality of naive IPS estimators and gains due to more advanced estimators (\cref{sec:advanced ips}). For each query $x \in \cX$, we generate a random vector $u_x \in [-1, 1]^4$. For each response $a \in \cA$, we generate a random vector $v_a \in [-1, 1]^4$. The feature vector of query $x$ and response $a$ is $\phi(x, a) = \mathrm{vec}(u_x v_a\T)$ and has length $d = 16$. The reward model parameter is sampled as $w_* \sim \cN(\mathbf{0}_d, 10^2 I_d)$.

We consider a parametric class of policies
\begin{align}
  \pi(a \mid x; \theta)
  = \frac{\exp[\phi(x, a)\T \theta]}{\sum_{i = 1}^L \exp[\phi(x, i)\T \theta]}\,.
  \label{eq:parametric policy}
\end{align}
The logging policy is specified by $\theta_0 = w_* + \varepsilon_0$ for $\varepsilon_0 \sim \cN(\mathbf{0}_d, 5^2 I_d)$. Because its per-dimension variance $5^2$ is much lower than that of $w_*$, the policy is likely to have a high reward. We evaluate $N = 5$ policies per run. Each evaluated policy is defined as $\theta_i = \theta_0 + \varepsilon_i$, where $\varepsilon_i \sim \cN(\mathbf{0}, \sigma_e^2 I)$ is its perturbation. The parameter $\sigma_e > 0$ defines how close the evaluated policy $\theta_i$ is to the logging policy $\theta_0$, and thus the hardness of the problem. We set $\sigma_e = 5$. The evaluation metric is the \emph{absolute error} $\frac{1}{N} \sum_{i = 1}^N |V(\theta_i) - \hat{V}(\theta_i)|$. All experiments are averaged over $50$ runs, where we randomize $w_*$, $\theta_0$, and all $\theta_i$. The default values of the parameters are $K = 2$ and $n = 3\,000$.

\begin{figure}[t]
  \centering
  \includegraphics[width=0.9\linewidth]{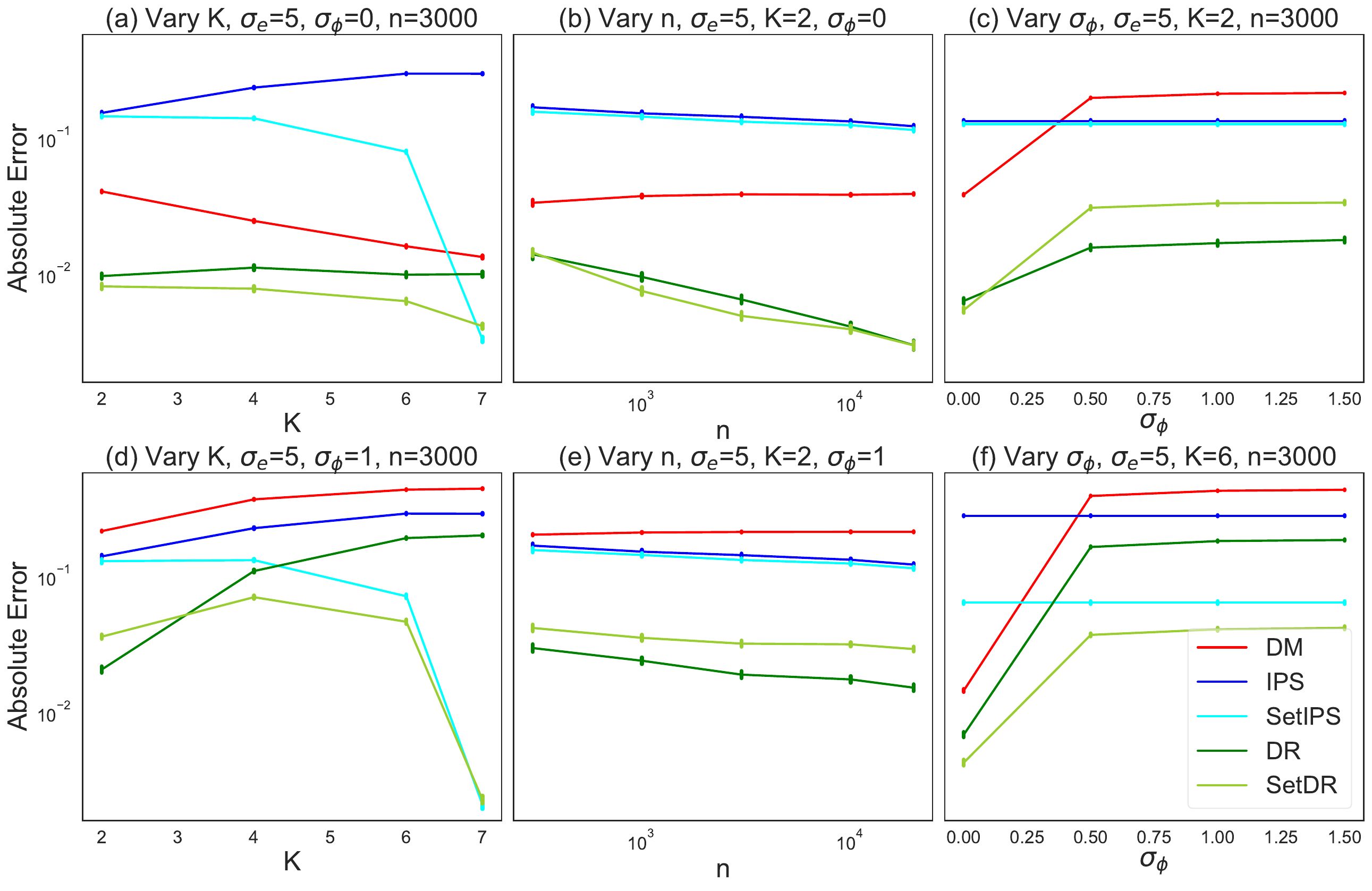}
  \vspace{-0.1in}
  \caption{Absolute error of estimated policy values as we vary $K$, $n$, and $\sigma_\phi$. }
  \label{fig:abserr_k2sig5phi0}
\end{figure}

Our results are reported in \cref{fig:abserr_k2sig5phi0}. In Figures \ref{fig:abserr_k2sig5phi0}a-b, the model is correctly specified. In \cref{fig:abserr_k2sig5phi0}a, we vary $K$. The best estimator is \setdr and the second best is DR. This does not come as a surprise since the superiority of DR estimators has long been recognized \citep{dudik14doubly}. The IPS estimator performs the worst at $K = L$. \setips leverages the full observability at $K = L$ (\cref{sec:advanced ips}) and thus performs well. In \cref{fig:abserr_k2sig5phi0}b, all estimators improve with more logged interactions $n$.

In the next experiments, the reward model can be misspecified. The misspecification is obtained by adding independent $\cN(0, \sigma_\phi^2)$ noise to $\phi(x, a)$ in the reward model in \eqref{eq:reward model}. In \cref{fig:abserr_k2sig5phi0}c, we vary $\sigma_\phi$ and observe its adverse effect on model-based estimators. Specifically, the DM becomes worse than IPS estimators, while the DR estimators remain more robust. In Figures \ref{fig:abserr_k2sig5phi0}d-e, we set $\sigma_\phi = 1$. In \cref{fig:abserr_k2sig5phi0}d, we vary $K$. We observe that the DM performs the worst and \setips is the second best estimator for $K > 4$. In \cref{fig:abserr_k2sig5phi0}e, all estimators improve with more logged interactions $n$. Finally, in \cref{fig:abserr_k2sig5phi0}f, we vary $\sigma_\phi$ and observe its adverse effect on model-based estimators, at $K = 6$.

\subsection{Relative Error}
\label{sec:relative error}

We left out RLHF and DPO from the absolute error comparison because they perform poorly. This is because they optimize a different notion of reward, which can be roughly viewed as the exponent in \eqref{eq:reward model}. However, since they are popular in optimization, they should be good in ranking policies. To test this, we conduct a relative comparison. Specifically, we repeat all experiments from \cref{sec:absolute error} but change the metric to the number of incorrectly ranked evaluated policies,
\begin{align*}
  \frac{2}{N (N - 1)} \sum_{i = 1}^N \sum_{j = i + 1}^N
  \I{\sgn(\hat{V}(\theta_i) - \hat{V}(\theta_j)) \neq \sgn(V(\theta_i) - V(\theta_j))}\,.
\end{align*}
We call this metric a \emph{relative error}. Our results are reported in \cref{fig:relerr_k2sig5phi0}. We observe two general trends. First, when the model is correctly specified, DM is among the best methods. Second, \setdr is always among the best methods.

\begin{figure}[t]
  \centering
  \includegraphics[width=0.9\linewidth]{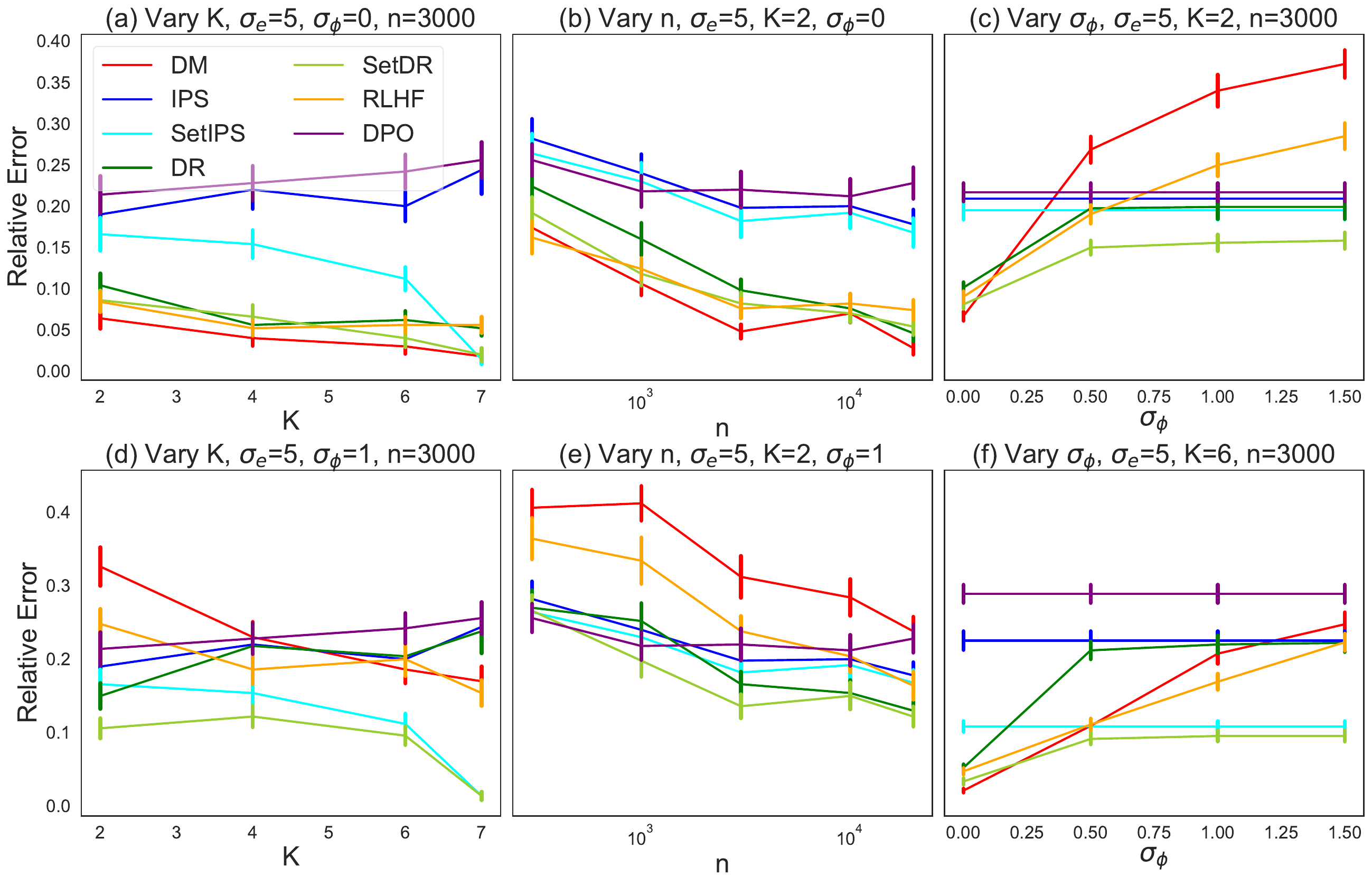}
  \vspace{-0.1in}
  \caption{Ranking error of estimated policy values as we vary $K$, $n$, and $\sigma_\phi$. }
  \label{fig:relerr_k2sig5phi0}
\end{figure}

\subsection{Large Language Models}
\label{sec:llms}

This experiment showcases our methods on a real-world Nectar dataset \citep{starling2023}. We take $500$ prompts from the dataset, each with $7$ responses generated by popular LLMs, and treat it as a logged dataset with $n = 500$ and $L = 7$. The feature vectors in \eqref{eq:reward model} are $768$-dimensional Instructor embeddings \citep{INSTRUCTOR} of the prompt and its response. We compute the propensities of individual responses using two large language models: $\pi_L(a \mid x)$ for LLama3 \citep{meta2024llama3} and $\pi_P(a \mid x)$ for Phi3 \citep{abdin2024phi}. We treat $\pi_P(a \mid x)$ as the logging policy and evaluate convex combinations of policies
\begin{align*}
  \pi_\alpha(a \mid x)
  = (1 - \alpha) \pi_P(a \mid x) + \alpha \pi_L(a \mid x)
\end{align*}
for $\alpha \in [0, 1]$. When $\alpha = 0$, the evaluated policy is the logging policy. When $\alpha = 1$, the evaluated policy is Llama3. The evaluation becomes progressively harder as $\alpha \to 1$. We use NVIDIA GeForce RTX 3090 GPU with 24GB RAM to load the large language models for $\pi_L(a \mid x)$ and $\pi_P(a \mid x)$. Phi3 requires less than 3GB RAM and LLama3 requires less than 7 GB RAM.

\begin{figure}[t]
  \centering
  \includegraphics[width=0.9\linewidth]{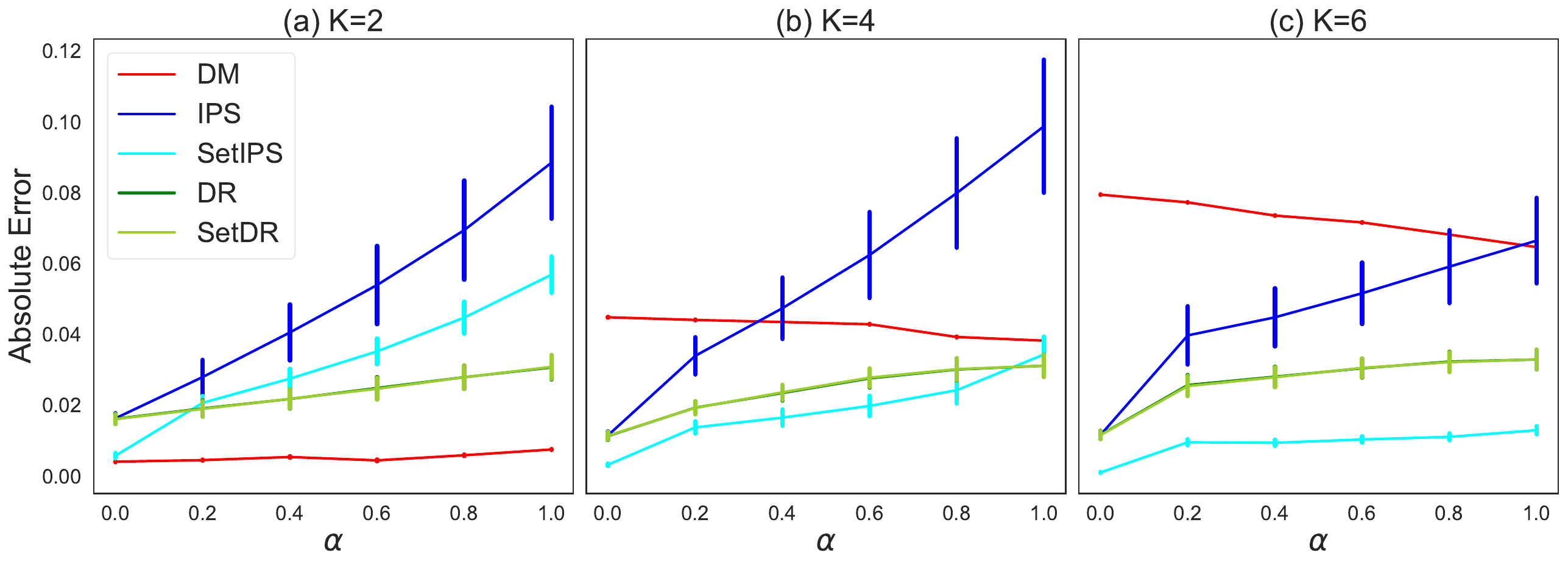}
  \caption{Evaluation of large language model policies on Nectar dataset by the absolute error. The parameter $\alpha \in [0, 1]$ interpolates between Phi3 and Llama3 policies.}
  \label{fig:expt3}
\end{figure}

Our results are reported in \cref{fig:expt3}. For $K > 2$, \setips performs the best. It is followed the DR and \setdr, which perform almost identically. The IPS estimator performs the worst for all $K$. The DM works well for $K = 2$. This is because the mean rewards of the best two responses are close to $0.5$ in most queries. As a result, they can be estimated well by a constant $0.5$, while the propensity scores in IPS methods only introduce unnecessary variance as $\alpha \to 1$.

\subsection{Policy Optimization}
\label{sec:policy optimization}

We experiment with the same problems as in \cref{sec:absolute error}. In the first problem, $K = 2$ and the logging policy is uniform, $\theta_0 = \mathbf{0}_d$. In the second problem, $K = 2$ and the logging policy has a high reward. We generate it as in \cref{sec:absolute error}. In the third problem, $K = 4$ and the logging policy is uniform. We optimize our estimators as described in \cref{sec:off-policy optimization}. RLHF and DPO are implemented as described in \cref{sec:related work}. We set $\gamma = 0.001$ and optimize the policies by Adam \citep{kingma15adam}. The logged dataset size is $n = 1\,000$. We choose these settings because they resulted in stable optimization results.

\begin{figure}[t]
  \centering
  \includegraphics[width=5.4in]{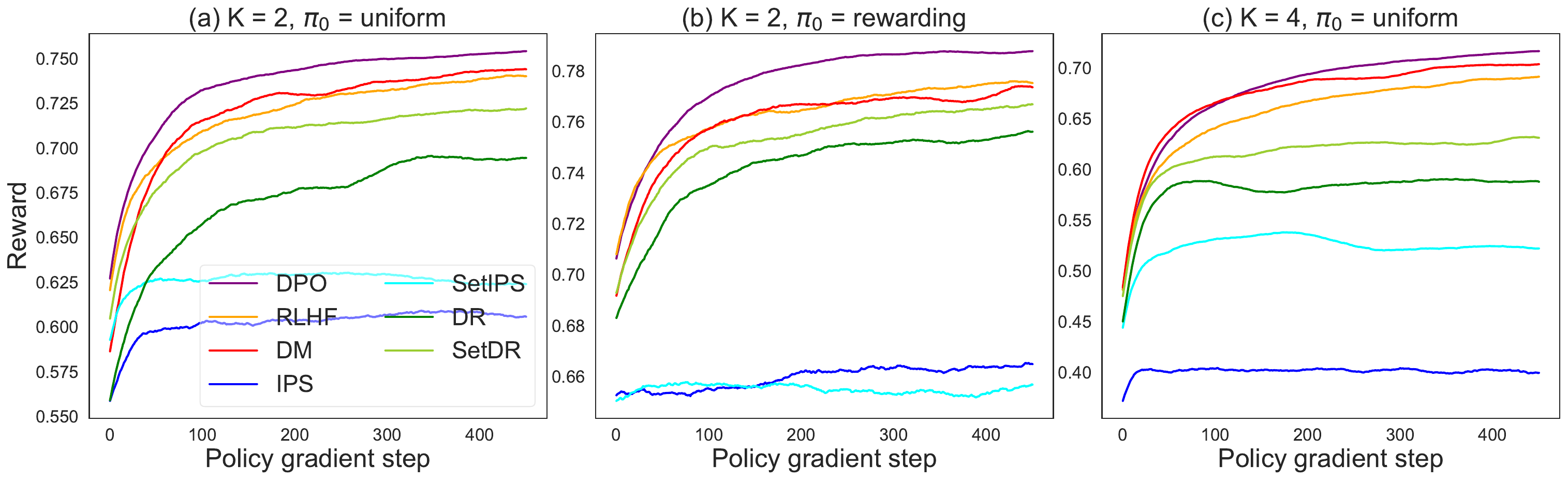}
  \vspace{-0.05in}
  \caption{Policy optimization of our estimators, together with RLHF and DPO. The plots are smoothed out by averaging over a window of $50$ steps.}
  \label{fig:policy optimization}
\end{figure}

In \cref{fig:policy optimization}, we report the values of all optimized policies in a single optimization run. The values are estimated using \eqref{eq:soft counter}. We observe two main trends. In all plots, DPO is among the best methods. Second, DM sometimes outperforms RLHF, while \setdr is the second best method. Interestingly, while our estimators are not designed for optimization, some of them perform well in this setting.

\section{Related Work}
\label{sec:related work}

The closest related works are off-policy evaluation and optimization for ranking \citep{joachims17unbiased,swaminathan17offpolicy,li18offline}. We differ from them in two key aspects. First, they consider absolute feedback on individual responses, such as clicks. Our feedback model is relative, the preference of one response over another. Second, these works assume that the responses in ranked lists are not affected by the logging policy; only their order is. Our logged lists contain $K$ responses out of $L > K$, and thus are affected by logging. This setting is motivated by the limited capacity of humans to provide relative feedback on too many responses, when $L \gg K$.

While we focus on evaluation, recent works on learning with human feedback primarily focused on optimization. We take RLHF \citep{ouyang22training} as an example. In RLHF, the goal is to optimize
\begin{align}
  \frac{1}{n} \sum_{t = 1}^n
  \Erv{a \sim \pi(\cdot \mid x_t; \theta)}{r_\textsc{rlhf}(x_t, a)} -
  \gamma d(\pi(\cdot \mid x_t; \theta), \pi_0(\cdot \mid x_t))\,,
  \label{eq:rlhf}
\end{align}
where $r_\textsc{rlhf}(x, a)$ in an estimate of the latent preference for response $a$ to query $x$, $d(p, q)$ is the KL divergence between distributions $p$ and $q$ with support $\cA$, $\gamma > 0$ is a tunable parameter, and $\theta$ is an optimized policy parameter. The reward model can be viewed as $r_\textsc{rlhf}(x, a) = \phi(x, a)\T w$, where $\phi$ is the same feature map as in \eqref{eq:reward model} and $w$ is learned as in \eqref{eq:reward model estimation}. The first term in \eqref{eq:rlhf} can be viewed as the value of policy $\theta$ in \eqref{eq:policy value}. Since $\phi(x, a)\T w_\textsc{rlhf}$ is latent, the value is not interpretable. This is the main conceptual difference from our work. We also study a plethora of policy value estimators, beyond what would be considered as the direct method in \eqref{eq:rlhf}.

One shortcoming of RLHF is the estimation of the latent reward model. This motivates DPO \citep{rafailov23direct}, where the latent reward is reparameterized using the optimized policy based on the structure of the maximized objective. For pairwise feedback, DPO maximizes
\begin{align}
  \sum_{t = 1}^n \log \sigma\left(
  \gamma \log\left(\frac{\pi(A_{t, *, 1} \mid x_t; \theta)}
  {\pi(A_{t, *, 1} \mid x_t; \theta_0)}\right) -
  \gamma \log\left(\frac{\pi(A_{t, *, 2} \mid x_t; \theta)}
  {\pi_0(A_{t, *, 2} \mid x_t)}\right)\right)\,,
  \label{eq:dpo}
\end{align}
where $A_{t, *} = (A_{t, *, 1}, A_{t, *, 2})$ and $\sigma(v) = 1 / (1 + \exp[v])$ is the sigmoid.

\section{Conclusions}
\label{sec:conclusions}

Learning from human feedback has been central to recent advances in artificial intelligence and machine learning. As more datasets with human feedback are collected, a natural question arises: do new datasets always need to be collected or can we reuse the old ones to estimate how a human would respond to a new policy? We propose both model-based and model-free estimators for this setting, analyze their unbiasedness, and show how to optimize them. We evaluate them empirically, on both synthetic and real-world datasets, and on both evaluation and optimization tasks. A natural direction to future work is studying other models of human feedback and reward \citep{train09discrete,benakiva18discrete}.

\bibliographystyle{plainnat}
\bibliography{biblio,Brano}

\begin{thebibliography}{34}
\providecommand{\natexlab}[1]{#1}
\providecommand{\url}[1]{\texttt{#1}}
\expandafter\ifx\csname urlstyle\endcsname\relax
  \providecommand{\doi}[1]{doi: #1}\else
  \providecommand{\doi}{doi: \begingroup \urlstyle{rm}\Url}\fi

\bibitem[Abdin et~al.(2024)Abdin, Jacobs, Awan, Aneja, Awadallah, Awadalla, Bach, Bahree, Bakhtiari, Behl, et~al.]{abdin2024phi}
Marah Abdin, Sam~Ade Jacobs, Ammar~Ahmad Awan, Jyoti Aneja, Ahmed Awadallah, Hany Awadalla, Nguyen Bach, Amit Bahree, Arash Bakhtiari, Harkirat Behl, et~al.
\newblock Phi-3 technical report: A highly capable language model locally on your phone.
\newblock \emph{arXiv preprint arXiv:2404.14219}, 2024.

\bibitem[Auer et~al.(1995)Auer, Cesa-Bianchi, Freund, and Schapire]{auer95gambling}
Peter Auer, Nicolo Cesa-Bianchi, Yoav Freund, and Robert Schapire.
\newblock Gambling in a rigged casino: The adversarial multi-armed bandit problem.
\newblock In \emph{Proceedings of the 36th Annual Symposium on Foundations of Computer Science}, pages 322--331, 1995.

\bibitem[Ben-Akiva and Lerman(2018)]{benakiva18discrete}
Moshe Ben-Akiva and Steven Lerman.
\newblock \emph{Discrete Choice Analysis: Theory and Application to Travel Demand}.
\newblock MIT Press, Cambridge, MA, 2018.

\bibitem[Bommasani et~al.(2021)]{bommasani21opportunities}
Rishi Bommasani et~al.
\newblock On the opportunities and risks of foundation models.
\newblock \emph{CoRR}, abs/2108.07258, 2021.
\newblock URL \url{https://arxiv.org/abs/2108.07258}.

\bibitem[Bottou et~al.(2013)Bottou, Peters, Quinonero-Candela, Charles, Chickering, Portugaly, Ray, Simard, and Snelson]{bottou13counterfactual}
Leon Bottou, Jonas Peters, Joaquin Quinonero-Candela, Denis Charles, Max Chickering, Elon Portugaly, Dipankar Ray, Patrice Simard, and Ed~Snelson.
\newblock Counterfactual reasoning and learning systems: The example of computational advertising.
\newblock \emph{Journal of Machine Learning Research}, 14\penalty0 (101):\penalty0 3207--3260, 2013.

\bibitem[Bradley and Terry(1952)]{bradley52rank}
Ralph~Allan Bradley and Milton Terry.
\newblock Rank analysis of incomplete block designs: I. the method of paired comparisons.
\newblock \emph{Biometrika}, 39\penalty0 (3-4):\penalty0 324--345, 1952.

\bibitem[Brown et~al.(2020)]{brown20language}
Tom Brown et~al.
\newblock Language models are few-shot learners.
\newblock In \emph{Advances in Neural Information Processing Systems 33}, 2020.

\bibitem[Dudik et~al.(2014)Dudik, Erhan, Langford, and Li]{dudik14doubly}
Miroslav Dudik, Dumitru Erhan, John Langford, and Lihong Li.
\newblock Doubly robust policy evaluation and optimization.
\newblock \emph{Statistical Science}, 29\penalty0 (4):\penalty0 485--511, 2014.

\bibitem[Hofmann et~al.(2016)Hofmann, Li, and Radlinski]{hofmann16online}
Katja Hofmann, Lihong Li, and Filip Radlinski.
\newblock Online evaluation for information retrieval.
\newblock \emph{Foundations and Trends in Information Retrieval}, 2016.

\bibitem[Horvitz and Thompson(1952)]{horwitz52generalization}
D.~G. Horvitz and D.~J. Thompson.
\newblock A generalization of sampling without replacement from a finite universe.
\newblock \emph{Journal of the American Statistical Association}, 47\penalty0 (260):\penalty0 663--685, 1952.

\bibitem[Ionides(2008)]{ionides08truncated}
Edward Ionides.
\newblock Truncated importance sampling.
\newblock \emph{Journal of Computational and Graphical Statistics}, 17\penalty0 (2):\penalty0 295--311, 2008.

\bibitem[Jiang and Li(2016)]{jiang2016doubly}
Nan Jiang and Lihong Li.
\newblock Doubly robust off-policy value evaluation for reinforcement learning.
\newblock In \emph{International conference on machine learning}, pages 652--661. PMLR, 2016.

\bibitem[Joachims et~al.(2017)Joachims, Swaminathan, and Schnabel]{joachims17unbiased}
Thorsten Joachims, Adith Swaminathan, and Tobias Schnabel.
\newblock Unbiased learning-to-rank with biased feedback.
\newblock In \emph{Proceedings of the 10th ACM International Conference on Web Search and Data Mining}, 2017.

\bibitem[Kingma and Ba(2015)]{kingma15adam}
Diederik Kingma and Jimmy Ba.
\newblock Adam: A method for stochastic optimization.
\newblock In \emph{Proceedings of the 3rd International Conference on Learning Representations}, 2015.

\bibitem[Lattimore and Szepesvari(2019)]{lattimore19bandit}
Tor Lattimore and Csaba Szepesvari.
\newblock \emph{Bandit Algorithms}.
\newblock Cambridge University Press, 2019.

\bibitem[Li et~al.(2010)Li, Chu, Langford, and Schapire]{li10contextual}
Lihong Li, Wei Chu, John Langford, and Robert Schapire.
\newblock A contextual-bandit approach to personalized news article recommendation.
\newblock In \emph{Proceedings of the 19th International Conference on World Wide Web}, 2010.

\bibitem[Li et~al.(2015)Li, Chen, Kleban, and Gupta]{li15counterfactual}
Lihong Li, Shunbao Chen, Jim Kleban, and Ankur Gupta.
\newblock Counterfactual estimation and optimization of click metrics in search engines: A case study.
\newblock In \emph{Proceedings of the 24th International Conference on World Wide Web}, 2015.

\bibitem[Li et~al.(2018)Li, Abbasi-Yadkori, Kveton, Muthukrishnan, Vinay, and Wen]{li18offline}
Shuai Li, Yasin Abbasi-Yadkori, Branislav Kveton, S.~Muthukrishnan, Vishwa Vinay, and Zheng Wen.
\newblock Offline evaluation of ranking policies with click models.
\newblock In \emph{Proceedings of the 24th ACM SIGKDD International Conference on Knowledge Discovery and Data Mining}, pages 1685--1694, 2018.

\bibitem[Luce(2005)]{luce05individual}
Robert~Duncan Luce.
\newblock \emph{Individual Choice Behavior: A Theoretical Analysis}.
\newblock Dover Publications, 2005.

\bibitem[Manning et~al.(2008)Manning, Raghavan, and Schutze]{manning08introduction}
Christopher Manning, Prabhakar Raghavan, and Hinrich Schutze.
\newblock \emph{Introduction to Information Retrieval}.
\newblock Cambridge University Press, 2008.

\bibitem[Meta(2024)]{meta2024llama3}
Meta.
\newblock Introducing meta llama 3: The most capable openly available llm to date.
\newblock 2024.
\newblock URL \url{https://ai.meta.com/blog/meta-llama-3/}.

\bibitem[Mirchandani et~al.(2023)Mirchandani, Xia, Florence, Ichter, Driess, Arenas, Rao, Sadigh, and Zeng]{mirchandani2023large}
Suvir Mirchandani, Fei Xia, Pete Florence, Brian Ichter, Danny Driess, Montserrat~Gonzalez Arenas, Kanishka Rao, Dorsa Sadigh, and Andy Zeng.
\newblock Large language models as general pattern machines.
\newblock \emph{arXiv preprint arXiv:2307.04721}, 2023.

\bibitem[Ouyang et~al.(2022)Ouyang, Wu, Jiang, Almeida, Wainwright, Mishkin, Zhang, Agarwal, Slama, Ray, Schulman, Hilton, Kelton, Miller, Simens, Askell, Welinder, Christiano, Leike, and Lowe]{ouyang22training}
Long Ouyang, Jeffrey Wu, Xu~Jiang, Diogo Almeida, Carroll Wainwright, Pamela Mishkin, Chong Zhang, Sandhini Agarwal, Katarina Slama, Alex Ray, John Schulman, Jacob Hilton, Fraser Kelton, Luke Miller, Maddie Simens, Amanda Askell, Peter Welinder, Paul Christiano, Jan Leike, and Ryan Lowe.
\newblock Training language models to follow instructions with human feedback.
\newblock In \emph{Advances in Neural Information Processing Systems 35}, 2022.

\bibitem[Plackett(1975)]{plackett75analysis}
Robin~Lewis Plackett.
\newblock The analysis of permutations.
\newblock \emph{Journal of the Royal Statistical Society: Series C (Applied Statistics)}, 24\penalty0 (2):\penalty0 193--202, 1975.

\bibitem[Rafailov et~al.(2023)Rafailov, Sharma, Mitchell, Manning, Ermon, and Finn]{rafailov23direct}
Rafael Rafailov, Archit Sharma, Eric Mitchell, Christopher Manning, Stefano Ermon, and Chelsea Finn.
\newblock Direct preference optimization: Your language model is secretly a reward model.
\newblock In \emph{Advances in Neural Information Processing Systems 36}, 2023.

\bibitem[Robins et~al.(1994)Robins, Rotnitzky, and Zhao]{robins94estimation}
James Robins, Andrea Rotnitzky, and Lue~Ping Zhao.
\newblock Estimation of regression coefficients when some regressors are not always observed.
\newblock \emph{Journal of the American Statistical Association}, 89\penalty0 (427):\penalty0 846--866, 1994.

\bibitem[Sekhari et~al.(2024)Sekhari, Sridharan, Sun, and Wu]{sekhari2024contextual}
Ayush Sekhari, Karthik Sridharan, Wen Sun, and Runzhe Wu.
\newblock Contextual bandits and imitation learning with preference-based active queries.
\newblock \emph{Advances in Neural Information Processing Systems}, 36, 2024.

\bibitem[Strehl et~al.(2010)Strehl, Langford, Li, and Kakade]{strehl10learning}
Alex Strehl, John Langford, Lihong Li, and Sham Kakade.
\newblock Learning from logged implicit exploration data.
\newblock In \emph{Advances in Neural Information Processing Systems 23}, 2010.

\bibitem[Su et~al.(2022)Su, Shi, Kasai, Wang, Hu, Ostendorf, Yih, Smith, Zettlemoyer, and Yu]{INSTRUCTOR}
Hongjin Su, Weijia Shi, Jungo Kasai, Yizhong Wang, Yushi Hu, Mari Ostendorf, Wen-tau Yih, Noah~A. Smith, Luke Zettlemoyer, and Tao Yu.
\newblock One embedder, any task: Instruction-finetuned text embeddings.
\newblock 2022.
\newblock URL \url{https://arxiv.org/abs/2212.09741}.

\bibitem[Swaminathan and Joachims(2015)]{swaminathan15counterfactual}
Adith Swaminathan and Thorsten Joachims.
\newblock Counterfactual risk minimization: Learning from logged bandit feedback.
\newblock In \emph{Proceedings of the 32nd International Conference on Machine Learning}, pages 814--823, 2015.

\bibitem[Swaminathan et~al.(2017)Swaminathan, Krishnamurthy, Agarwal, Dudik, Langford, Jose, and Zitouni]{swaminathan17offpolicy}
Adith Swaminathan, Akshay Krishnamurthy, Alekh Agarwal, Miro Dudik, John Langford, Damien Jose, and Imed Zitouni.
\newblock Off-policy evaluation for slate recommendation.
\newblock In \emph{Advances in Neural Information Processing Systems 30}, 2017.

\bibitem[Train(2009)]{train09discrete}
Kenneth Train.
\newblock \emph{Discrete Choice Methods with Simulation}.
\newblock Cambridge University Press, New York, NY, 2009.

\bibitem[Zhu et~al.(2023{\natexlab{a}})Zhu, Frick, Wu, Zhu, and Jiao]{starling2023}
Banghua Zhu, Evan Frick, Tianhao Wu, Hanlin Zhu, and Jiantao Jiao.
\newblock Starling-7b: Improving llm helpfulness \& harmlessness with rlaif, November 2023{\natexlab{a}}.

\bibitem[Zhu et~al.(2023{\natexlab{b}})Zhu, Jiao, and Jordan]{zhu23principled}
Banghua Zhu, Jiantao Jiao, and Michael Jordan.
\newblock Principled reinforcement learning with human feedback from pairwise or {$K$}-wise comparisons.
\newblock \emph{CoRR}, abs/2301.11270, 2023{\natexlab{b}}.
\newblock URL \url{https://arxiv.org/abs/2301.11270}.

\end{thebibliography}

\clearpage
\onecolumn
\appendix

\section{Technical Lemmas}
\label{sec:technical lemmas}

\begin{lemma}
\label{lem:counter} Let $\hat{V}(\pi)$ be defined as in \eqref{eq:counter}. Then $\E{\hat{V}(\pi)} = V(\pi)$.
\end{lemma}
\begin{proof}
The proof follows from a sequence of identities,
\begin{align*}
  \E{\hat{V}(\pi)}
  & = \frac{1}{n} \sum_{t = 1}^n
  \Erv{A_t \sim \pi(\cdot \mid x_t), \, A_{t, *} \sim \cdot \mid x_t, A_t}
  {\I{a_{t, 1} = a_{t, *, 1}}} \\
  & = \frac{1}{n} \sum_{t = 1}^n
  \Erv{A_t \sim \pi(\cdot \mid x_t)}
  {\condprob{a_{t, 1} = a_{t, *, 1}}{x_t, A_t}} \\
  & = \frac{1}{n} \sum_{t = 1}^n
  \Erv{A_t \sim \pi(\cdot \mid x_t)}{r(x_t, A_t)}
  = V(\pi)\,.
\end{align*}
This completes the proof.
\end{proof}

\begin{lemma}
\label{lem:ips} Let $\hat{V}_\textsc{ips}(\pi)$ be defined as in \eqref{eq:ips}. Then $\E{\hat{V}_\textsc{ips}(\pi)} = V(\pi)$.
\end{lemma}
\begin{proof}
The proof is similar to \cref{lem:counter}, with the addition of propensity scores. In particular,
\begin{align*}
  \E{\hat{V}_\textsc{ips}(\pi)}
  & = \frac{1}{n} \sum_{t = 1}^n
  \Erv{A_t \sim \pi_0(\cdot \mid x_t), \, A_{t, *} \sim \cdot \mid x_t, A_t}
  {\frac{\pi(A_t \mid x_t)}{\pi_0(A_t \mid x_t)}
  \I{a_{t, 1} = a_{t, *, 1}}} \\
  & = \frac{1}{n} \sum_{t = 1}^n
  \Erv{A_t \sim \pi_0(\cdot \mid x_t)}
  {\frac{\pi(A_t \mid x_t)}{\pi_0(A_t \mid x_t)}
  \condprob{a_{t, 1} = a_{t, *, 1}}{x_t, A_t}} \\
  & = \frac{1}{n} \sum_{t = 1}^n
  \Erv{A_t \sim \pi(\cdot \mid x_t)}
  {\condprob{a_{t, 1} = a_{t, *, 1}}{x_t, A_t}} \\
  & = \frac{1}{n} \sum_{t = 1}^n
  \Erv{A_t \sim \pi(\cdot \mid x_t)}{r(x_t, A_t)}
  = V(\pi)\,.
\end{align*}
This completes the proof.
\end{proof}

\begin{lemma}
\label{lem:dr} Let $\hat{V}_\textsc{dr}(\pi)$ be defined as in \eqref{eq:dr}. Then $\E{\hat{V}_\textsc{dr}(\pi)} = V(\pi)$ when the DM is unbiased or the propensities in the IPS estimator are correctly specified.
\end{lemma}
\begin{proof}
When the reward model is correctly specified,
\begin{align*}
  \condE{\I{a_{t, 1} = a_{t, *, 1}} - r(x_t, A_t; \hat{w})}{x_t, A_t}
  = r(x_t, A_t) - r(x_t, A_t; \hat{w})
  = 0\,,
\end{align*}
and the first term in \eqref{eq:dr} vanishes. This proves the first claim.

When the propensities are correctly specified,
\begin{align*}
  \Erv{A_t \sim \pi_0(\cdot \mid x_t)}
  {\frac{\pi(A_t \mid x_t)}{\pi_0(A_t \mid x_t)}{r(x_t, A_t; \hat{w})}}
  = \Erv{A_t \sim \pi(\cdot \mid x_t)}{r(x_t, A_t; \hat{w})}\,.
\end{align*}
In this case, $r(x_t, A_t; \hat{w})$ and $\hat{V}_\textsc{dm}(\pi)$ cancel out. This proves the second claim.
\end{proof}

\begin{lemma}
\label{lem:soft counter} Let $\hat{V}_\textsc{set}(\pi)$ be defined as in \eqref{eq:soft counter}. Then $\E{\hat{V}_\textsc{set}(\pi)} = V(\pi)$.
\end{lemma}
\begin{proof}
The proof follows from introducing a conditional expectation,
\begin{align*}
  \E{\hat{V}_\textsc{set}(\pi)}
  & = \frac{1}{n} \sum_{t = 1}^n
  \Erv{S_t \sim \nu(\cdot \mid x_t), \, A_{t, *} \sim \cdot \mid x_t, S_t}
  {\pi(a_{t, *, 1} \mid S_t, x_t)} \\
  & = \frac{1}{n} \sum_{t = 1}^n
  \Erv{S_t \sim \nu(\cdot \mid x_t)}
  {\condE{\pi(a_{t, *, 1} \mid S_t, x_t)}{x_t, S_t}}
  = V(\pi)\,.
\end{align*}
This completes the proof.
\end{proof}

\begin{lemma}
\label{lem:set ips} Let $\hat{V}_\textsc{set-ips}(\pi)$ be defined as in \eqref{eq:set ips}. Then $\E{\hat{V}_\textsc{set-ips}(\pi)} = V(\pi)$.
\end{lemma}
\begin{proof}
The proof is similar to \cref{lem:soft counter}, with the addition of propensity scores. In particular,
\begin{align*}
  \E{\hat{V}_\textsc{set-ips}(\pi)}
  & = \frac{1}{n} \sum_{t = 1}^n
  \Erv{S_t \sim \nu_0(\cdot \mid x_t), \, A_{t, *} \sim \cdot \mid x_t, S_t}
  {\frac{\nu(S_t \mid x_t)}{\nu_0(S_t \mid x_t)}
  \pi(a_{t, *, 1} \mid S_t, x_t)} \\
  & = \frac{1}{n} \sum_{t = 1}^n
  \Erv{S_t \sim \nu_0(\cdot \mid x_t)}
  {\frac{\nu(S_t \mid x_t)}{\nu_0(S_t \mid x_t)}
  \condE{\pi(a_{t, *, 1} \mid S_t, x_t)}{x_t, S_t}} \\
  & = \frac{1}{n} \sum_{t = 1}^n
  \Erv{S_t \sim \nu(\cdot \mid x_t)}
  {\condE{\pi(a_{t, *, 1} \mid S_t, x_t)}{x_t, S_t}}
  = V(\pi)\,.
\end{align*}
This completes the proof.
\end{proof}

\begin{lemma}
\label{lem:set dr} Consider $\hat{V}_\textsc{set-dr}(\pi)$ in \eqref{eq:set dr}. Then $\E{\hat{V}_\textsc{set-dr}(\pi)} = V(\pi)$ when the DM is unbiased or the propensities in the IPS estimator are correctly specified.
\end{lemma}
\begin{proof}
We start by noting that
\begin{align*}
  \condE{\pi(a_{t, *, 1} \mid S_t, x_t)}{x_t, S_t}
  = \sum_{A \in \Pi(S_t)} \frac{\pi(A \mid x_t)}{\nu(S_t \mid x_t)} \, r(x_t, A)
\end{align*}
and
\begin{align*}
  r(x_t, S_t; \pi, \hat{w})
  = \sum_{A \in \Pi(S_t)} \frac{\pi(A \mid x_t)}{\nu(S_t \mid x_t)} r(x_t, A; \hat{w})\,.
\end{align*}
Therefore, when the reward model is correctly specified,
\begin{align*}
  \condE{\pi(a_{t, *, 1} \mid S_t, x_t) - r(x_t, S_t; \pi, \hat{w})}{x_t, S_t}
  = 0\,,
\end{align*}
and the first term in \eqref{eq:set dr} vanishes. This proves the first claim.

When the propensities are correctly specified,
\begin{align*}
  \Erv{S_t \sim \nu_0(\cdot \mid x_t)}
  {\frac{\nu(S_t \mid x_t)}{\nu_0(S_t \mid x_t)}{r(x_t, S_t; \pi, \hat{w})}}
  & = \Erv{S_t \sim \nu(\cdot \mid x_t)}{r(x_t, S_t; \pi, \hat{w})} \\
  & = \Erv{S_t \sim \nu(\cdot \mid x_t)}
  {\sum_{A \in \Pi(S_t)} \frac{\pi(A \mid x_t)}{\nu(S_t \mid x_t)} r(x_t, A; \hat{w})} \\
  & = \Erv{A \sim \pi(\cdot \mid x_t)}{r(x_t, A; \hat{w})}\,.
\end{align*}
In this case, $r(x_t, S_t; \pi, \hat{w})$ and $\hat{V}_\textsc{dm}(\pi)$ cancel out. This proves the second claim.
\end{proof}

\section{Gradients}
\label{sec:gradients}

The gradient of the regularizer is derived as follows. First, we fix interaction $t \in [n]$. Since $t$ is fixed, $x_t$ is fixed, and thus we can write $\pi(\cdot; \theta)$ instead $\pi(\cdot; x_t, \theta)$. Then, using basic algebra, we get
\begin{align*}
  & \nabla d(\pi(\cdot; \theta), \pi_0) \\
  & = \sum_{a \in \cA} \nabla [\pi(a; \theta)
  (\log \pi(a; \theta) - \log \pi_0(a))] \\
  & = \sum_{a \in \cA} [\nabla \pi(a; \theta)] \log \pi(a; \theta) +
  \pi(a; \theta) \nabla \log \pi(a; \theta) -
  [\nabla \pi(a; \theta)] \log \pi_0(a) \\
  & = \sum_{a \in \cA} \pi(a; \theta) [\nabla \log \pi(a; \theta)]
  [1 + \log \pi(a; \theta) - \log \pi_0(a)]\,.
\end{align*}
This implies that for $a \sim \pi(\cdot; \theta)$,
\begin{align*}
  \nabla d(\pi(\cdot; \theta), \pi_0)
  = \E{(\nabla \log \pi(a; \theta)) (1 + \log \pi(a; \theta) - \log \pi_0(a))}\,.
\end{align*}
Therefore, $(\nabla \log \pi(a; \theta)) (1 + \log \pi(a; \theta) - \log \pi_0(a))$ is an unbiased single-sample estimate of the gradient.

The DM gradient is computed as follows. We sample $\tilde{A}_t \sim \pi(\cdot \mid x_t; \theta)$ and then use
\begin{align}
  \frac{1}{n} \sum_{t = 1}^n
  \nabla \log \pi(\tilde{A}_t \mid x_t; \theta) \,
  r(x_t, \tilde{A}_t; \hat{w})\,.
  \label{eq:dm gradient}
\end{align}
This is an unbiased single-sample estimate of $\nabla \hat{V}_\textsc{dm}(\pi)$.

The IPS gradient is computed directly as
\begin{align}
  \nabla \hat{V}_\textsc{ips}(\pi)
  = \frac{1}{n} \sum_{t = 1}^n
  \frac{\nabla \pi(A_t \mid x_t; \theta)}{\pi_0(A_t \mid x_t)}
  \I{a_{t, 1} = a_{t, *, 1}}\,.
  \label{eq:ips gradient}
\end{align}
Combining the above, the DR gradient can be computed as follows. We sample $\tilde{A}_t \sim \pi(\cdot \mid x_t; \theta)$ and then use
\begin{align}
  \frac{1}{n} \sum_{t = 1}^n
  \frac{\nabla \pi(A_t \mid x_t; \theta)}{\pi_0(A_t \mid x_t)}
  (\I{a_{t, 1} = a_{t, *, 1}} - r(x_t, A_t; \hat{w})) +
  \nabla \hat{V}_\textsc{dm}(\pi)\,.
  \label{eq:dr gradient}
\end{align}
This is an unbiased single-sample estimate of $\nabla \hat{V}_\textsc{dr}(\pi)$.



\end{document}